\documentclass{article} 
\usepackage{iclr2025_conference,times}

\usepackage{graphicx}
\usepackage{float}
\usepackage{subcaption}
\usepackage{booktabs}
\usepackage{multirow}
\usepackage{algorithm}
\usepackage{algorithmic}
\usepackage{hyperref}
\usepackage{url}
\usepackage{xcolor}
\usepackage{amsmath}
\usepackage{amssymb}
\usepackage{amsfonts}
\usepackage{mathtools}
\usepackage{amsthm}
\usepackage{bm}
\usepackage{nicefrac}
\usepackage{ragged2e}
\usepackage[capitalize,noabbrev]{cleveref}

\newcommand{\rkeep}{\rho_{\mathrm{keep}}}
\newcommand{\eh}{\chi_{\phi}}

\theoremstyle{plain}
\newtheorem{theorem}{Theorem}[section]
\newtheorem{proposition}[theorem]{Proposition}

\newtheorem{corollary}[theorem]{Corollary}
\newtheorem{insight}[theorem]{Insight}
\theoremstyle{definition}
\newtheorem{definition}[theorem]{Definition}

\theoremstyle{remark}
\newtheorem{remark}[theorem]{Remark}

\title{Nonlinearity-Aware LoRA: Structured Gate Adaptation under Low-Rank Constraints}

\author{\textbf{Shuai Yuan$^{1}$, Sudong Cai$^{2}\thanks{Corresponding author: Sudong Cai.}$ , Bingzhi Chen$^{1}$, 
Shuyuan Zheng$^{3}$ , Chuan Xiao$^{3}$}\\
\textbf{Makoto Onizuka$^{3}$ , Rui Mao$^{4}$}\\
$^{1}$Beijing Institute of Technology, Zhuhai ;
$^{2}$The Hong Kong Polytechnic University;\\
$^{3}$The University of Osaka;
$^{4}$Shenzhen University\\
\texttt{yuanshuai@bit.edu.cn; sudong.cai@polyu.edu.hk}
}

\iclrfinalcopy 
\begin{document}

\maketitle

\begin{abstract}
Low-rank adaptation (LoRA) is commonly justified as an update-space approximation to full fine-tuning, yet this view is incomplete for self-gated Transformer feed-forward networks (FFNs).  Drawing on the decision-making interpretation of activation as selective feature recalibration, we show that a low-rank residual can change not only the projected features but also the nonlinear \emph{selection weights} that determine which channels contribute to the FFN output.  We formalize this effect as \emph{selection misalignment} and connect it to the local \emph{effective homogeneity} of self-gated activations: the output behavior of a channel is governed by the gate pre-activation through a data-dependent homogeneity profile.  This motivates an effective-homogeneity-aware principle for PEFT: low-rank updates should spend capacity on gate channels whose nonlinear states remain responsive and should shape, rather than suppress, the temporal evolution of selection.  We propose \textbf{NA-LoRA} (Nonlinearity-Aware LoRA), a training-only method that implements this principle with two lightweight mechanisms: a derivative-based temporal-importance mask for gate-related LoRA updates, and an activation-specific step-scaling rule when a meaningful coarse effective-homogeneity partition is available.  NA-LoRA adds no auxiliary loss and incurs no inference-time overhead.  Across language-model fine-tuning and vision-language transfer benchmarks, NA-LoRA improves over vanilla LoRA and is competitive with or better than strong PEFT variants across task metrics, often approaching or exceeding full fine-tuning in the reported settings.
\end{abstract}

\section{Introduction}
\label{sec:intro}

Low-Rank Adaptation (LoRA)~\cite{hu2022lora} has become a standard parameter-efficient fine-tuning (PEFT) method for large language models.  By freezing the pretrained backbone and learning a low-rank update, LoRA reduces memory and optimization cost while retaining much of the flexibility of full fine-tuning.  Nevertheless, a persistent gap between LoRA and full fine-tuning remains in many settings~\cite{ding2023parameter}.  Recent variants reduce this gap through better rank allocation, initialization, update parameterization, or optimizer design~\cite{kalajdzievski2023rank,zhang2023adaptive,meng2024pissa,liu2024dora}.  Most of these methods share an implicit premise: if the low-rank update is well aligned with the full update in weight, gradient, or subspace geometry, then the adapted model should behave similarly.

This premise is useful but incomplete for Transformer feed-forward networks (FFNs), especially gated variants such as SwiGLU.  In a gated FFN, behavior is not determined solely by the linear projections.  The gate activation also converts pre-activation scores into multiplicative decision weights that reweight channel utilities.  Thus, a low-rank update can be close to a target update in parameter space while still inducing a different pattern of nonlinear channel selection.  This discrepancy is particularly important when small gate perturbations move channels across responsive, saturated, or positive non-amplifying regimes, because these regime changes alter how features are emphasized or suppressed by the FFN.

We revisit LoRA through the decision-making view of neural activation developed in the MCDM activation literature~\cite{caiIIEURethinkingNeural2023,caiAdaShiftLearningDiscriminative2024,caiPrincipledFlexibleScaling2025}.  In this view, an affine or linear projection provides an evaluation score, and the activation acts as a selective recalibrator that transforms the score into a decision weight.  Adapting this lens from activation design to PEFT leads to a different diagnosis of the LoRA gap: beyond approximating a full update, an adapter should account for the \emph{selection behavior} induced by the existing nonlinear gates.  We call the resulting failure mode \emph{selection misalignment}; Figure~\ref{fig:motivation} illustrates why a small weight-space residual can still change activation-induced channel selection.

Our analysis decomposes this misalignment into a view mismatch, a residual-through-view term, and a selection-mismatch term.  We further connect the selection term to the local \emph{effective homogeneity} of self-gated activations, which describes how a channel's nonlinear output changes under local score and input scaling.  This profile reveals whether gate changes are suppressive, responsive, or already in a positive non-amplifying state.  Importantly, these regions are not fixed engineering constants; they are induced by the activation-specific effective-homogeneity curve.  This suggests a design principle for low-rank adaptation: LoRA updates should not only fit gradients, but should also allocate capacity and update speed according to the nonlinear state of FFN gates.

We propose \textbf{NA-LoRA} (Nonlinearity-Aware LoRA), a training-only method that implements this principle without auxiliary alignment losses or inference-time overhead.  First, a \emph{temporal-importance channel mask} uses gate-derivative sensitivity to identify channels whose gate states are consistently responsive, allocating limited low-rank capacity to channels that can most affect selection behavior.  Second, a \emph{homogeneity-dynamics step scaling} rule uses activation-specific coarse regime statistics when they are non-degenerate, shaping update speed without unstructured amplification of the entire FFN adapter.  Attention projections remain standard LoRA; NA-LoRA applies nonlinearity-aware control only where FFN gates induce selection dynamics.

\noindent\textbf{Contributions.}
(1) We identify \emph{selection misalignment} as a behavior-space source of the LoRA gap in self-gated FFNs, complementing existing weight-, gradient-, and rank-centric analyses.
(2) We formalize a decision-making view of LoRA adaptation and derive effective-homogeneity diagnostics that connect gate sensitivity, channel utility, and structured temporal modulation.
(3) We introduce \textbf{NA-LoRA}, combining sensitivity-based temporal channel masking with activation-specific LoRA step scaling when a non-degenerate regime statistic is available.
(4) We validate NA-LoRA on language-model fine-tuning and CLIP transfer benchmarks, showing consistent gains over vanilla LoRA and competitive or improved performance relative to strong PEFT baselines with zero inference-time overhead.

\begin{figure}[t]
    \centering
    \includegraphics[width=\linewidth]{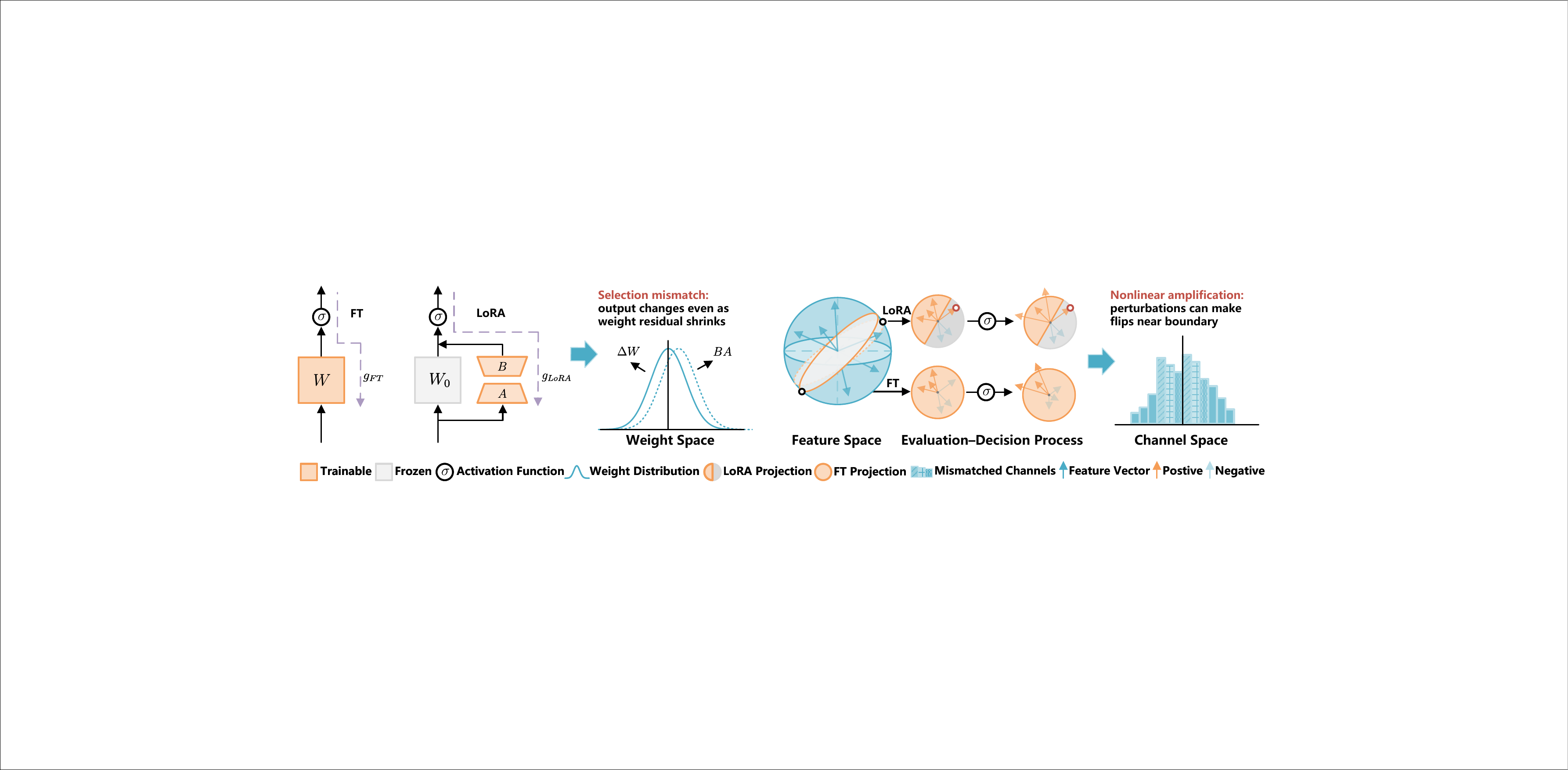}
    \caption{\textbf{Motivation.} Weight-space alignment does not guarantee activation-induced selection alignment in self-gated FFNs.  NA-LoRA targets the training-time gate dynamics that control channel usage.}
    \label{fig:motivation}
\end{figure}

\section{Related Work}
\label{sec:related}

\noindent\textbf{Low-rank adaptation.}
Parameter-efficient fine-tuning adapts pretrained models by updating a small set of trainable parameters, e.g., adapters, prompts, or low-rank factors~\cite{houlsby2019parameter,lester2021power,ding2023parameter,hu2022lora}.  LoRA~\cite{hu2022lora} is particularly widely used because it freezes the backbone and represents each update by a low-rank product.  Recent variants improve LoRA through rank allocation, initialization, update reparameterization, or gradient/optimization alignment~\cite{kalajdzievski2023rank,zhang2023adaptive,meng2024pissa,liu2024dora,wang2024loraga,wang2024lorapro,heGoRAGradientdrivenAdaptive2025}.  These methods mainly reason in weight, gradient, or subspace geometry; our focus is the nonlinear behavior induced by self-gated FFNs after such updates.

\noindent\textbf{Activation as selection.}
Self-gated activations such as SiLU/GELU transform pre-activation scores into input-dependent recalibration weights.  A related decision-making line interprets the affine--activation pipeline as feature evaluation followed by decision weighting, and uses this view to improve activation design or adaptive scaling~\cite{caiIIEURethinkingNeural2023,caiAdaShiftLearningDiscriminative2024,caiPrincipledFlexibleScaling2025}.  Recent nonlinear adapter work also suggests that adding nonlinear mappings to LoRA can improve PEFT expressivity~\cite{dong2025aurora}.  NA-LoRA differs in that it does not redesign the activation or add inference-time nonlinear modules; it uses the selection view to modulate how existing FFN gates evolve during LoRA training.

\noindent\textbf{Our position.}
Existing LoRA methods narrow the gap to full fine-tuning by improving the low-rank update itself.  We study a complementary failure mode: a low-rank update may be plausible in parameter space yet alter the gate-induced score-to-weight recalibration that determines channel usage.  NA-LoRA therefore treats FFN gates as output-determining selection units and applies training-only, history-aware modulation to their low-rank updates.

\section{Problem: Selection Misalignment in LoRA}
\label{sec:problem}

LoRA is usually analyzed as a low-rank approximation to a target update.  For a self-gated FFN, however, the output is determined jointly by the feature view, the gate score, and the activation-induced decision weight.  A small residual in a gate projection can therefore change the nonlinear rule by which channels are selected, even if the residual is small under a conventional matrix norm.  We call this behavior-space discrepancy \emph{selection misalignment}.

The target full-fine-tuning update used below is a conceptual reference, not an additional signal required by NA-LoRA.  It exposes why update-space alignment alone is insufficient and motivates training-time control that can be estimated from the observed gate dynamics of the LoRA model itself.

\subsection{Selection view of self-gated FFNs}
\label{subsec:selection_view}

Consider a SwiGLU-style FFN block
\begin{equation}
\label{eq:ffn_block}
u = W_u x,\qquad z = W_g x,\qquad h = \phi(z)\odot u,\qquad y=W_dh,
\end{equation}
where $u$ is the feature view and $z$ is the gate pre-activation.  We omit biases and normalization from the notation; the analysis applies locally after the normalized representation has been formed, and LoRA scaling factors are absorbed into the low-rank matrices.

\begin{definition}[Selection weights]
\label{def:selection_weights}
For a self-gated nonlinearity of the form $\phi(z)=\varrho(z)z$, define the channel-wise selection weight
\begin{equation}
\label{eq:selection_weight_def}
\phi(z)=s(z)\odot z,\qquad s_i(z)=\varrho(z_i)\ge 0 .
\end{equation}
At isolated points where the ratio form is ambiguous, $s$ is defined by continuity.  For SiLU, for example, $s_i(z)=\sigma(z_i)$.
\end{definition}

\begin{insight}[Activation as selective recalibration]
\label{ins:selection_recalibration}
Following the MCDM interpretation of activation, $z$ acts as an evaluation score, $s(z)$ acts as a decision-weighting rule, and $u$ provides the candidate feature view.  The FFN channel utility is therefore
\begin{equation}
\label{eq:utility_def}
\mathcal{U}(z,u):=\phi(z)\odot u = \big(s(z)\odot z\big)\odot u.
\end{equation}
Selection alignment means preserving this score-to-weight recalibration, not merely approximating the linear projection that produces $z$.
\end{insight}

\subsection{Weight-space fit need not imply selection alignment}
\label{subsec:selection_mismatch}

Let the target and LoRA-adapted gate projections be
\begin{equation}
W_{g,\mathrm{tar}}=W_{g,0}+\Delta W_g,
\qquad
W_{g,\mathrm{lora}}=W_{g,0}+\widehat{\Delta W}_g,
\qquad
\widehat{\Delta W}_g=B_gA_g,
\end{equation}
and define the gate residual $E_g:=\Delta W_g-\widehat{\Delta W}_g$.  Similarly, define the feature-view residual $E_u:=\Delta W_u-\widehat{\Delta W}_u$ for the up projection.  For an input $x$,
\begin{equation}
z_{\mathrm{tar}}=W_{g,\mathrm{tar}}x,
\qquad
z_{\mathrm{lora}}=W_{g,\mathrm{lora}}x,
\qquad
\Delta z=z_{\mathrm{tar}}-z_{\mathrm{lora}}=E_gx,
\end{equation}
and analogously $\Delta u=u_{\mathrm{tar}}-u_{\mathrm{lora}}=E_ux$.

\begin{proposition}[Exact decomposition of FFN mismatch]
\label{prop:exact_selection_decomp}
For a self-gated FFN with $\phi(z)=s(z)\odot z$, the channel-utility mismatch admits the exact decomposition
\begin{equation}
\label{eq:exact_selection_decomp}
\begin{split}
&\mathcal{U}(z_{\mathrm{tar}},u_{\mathrm{tar}})
-\mathcal{U}(z_{\mathrm{lora}},u_{\mathrm{lora}}) \\
&=
\underbrace{\big(s(z_{\mathrm{tar}})\odot \Delta z\big)\odot u_{\mathrm{lora}}}_{\textnormal{residual through the current view}}
+
\underbrace{\big(s(z_{\mathrm{tar}})-s(z_{\mathrm{lora}})\big)\odot z_{\mathrm{lora}}\odot u_{\mathrm{lora}}}_{\textnormal{selection mismatch}}
+
\underbrace{\phi(z_{\mathrm{tar}})\odot \Delta u}_{\textnormal{view mismatch}} .
\end{split}
\end{equation}
\end{proposition}

The second term is the object ignored by a purely weight-space view.  It is small only when the low-rank residual induces little change in the decision weights, or when the affected channels have little utility.  Thus, what matters is not just $\|E_g\|$, but how $E_gx$ aligns with gate-sensitive and high-utility directions on the data distribution.

\begin{corollary}[Selection change is a behavior-space control variable]
\label{cor:selection_amplification}
For any $\ell_p$ norm compatible with the elementwise product,
\begin{equation}
\label{eq:selection_amplification_bound}
\begin{split}
&\|\mathcal{U}(z_{\mathrm{tar}},u_{\mathrm{tar}})
-\mathcal{U}(z_{\mathrm{lora}},u_{\mathrm{lora}})\|_p \\
&\le
\|u_{\mathrm{lora}}\|_{\infty}
\Big(
\|s(z_{\mathrm{tar}})\|_{\infty}\|\Delta z\|_p
+
\|z_{\mathrm{lora}}\|_{\infty}\|s(z_{\mathrm{tar}})-s(z_{\mathrm{lora}})\|_p
\Big)
+
\|\phi(z_{\mathrm{tar}})\|_{\infty}\|\Delta u\|_p.
\end{split}
\end{equation}
Consequently, a useful behavior-space bound requires controlling target-misaligned selection change $\|s(z_{\mathrm{tar}})-s(z_{\mathrm{lora}})\|_p$ in addition to the projection residuals.
\end{corollary}

\begin{proposition}[Sensitivity-weighted selection mismatch]
\label{prop:sensitivity_weighted_mismatch}
Assume $s$ is differentiable on the segment between $z_{\mathrm{lora}}$ and $z_{\mathrm{tar}}$.  For each channel $c$, there exists an intermediate point $\bar z_c$ on this segment such that
\begin{equation}
\label{eq:sensitivity_weighted_mismatch}
\big(s(z_{\mathrm{tar},c})-s(z_{\mathrm{lora},c})\big)z_{\mathrm{lora},c}u_{\mathrm{lora},c}
=
 s'(\bar z_c)\,\Delta z_c\,z_{\mathrm{lora},c}u_{\mathrm{lora},c}.
\end{equation}
Thus the squared selection-mismatch energy is governed by
\begin{equation}
\label{eq:sensitivity_weighted_residual}
\sum_c |s'(\bar z_c)|^2 |z_{\mathrm{lora},c}|^2 |u_{\mathrm{lora},c}|^2 |\Delta z_c|^2.
\end{equation}
\end{proposition}

This expression gives the rank-allocation principle behind NA-LoRA: under a fixed low-rank budget, updates should prioritize channels that are both responsive to gate perturbations and repeatedly useful for the FFN output.  Directly estimating $s'(\bar z_c)$ and the target residual is impractical, so we next introduce a gate statistic that is available during ordinary LoRA training.

\subsection{Gate sensitivity and effective homogeneity}
\label{subsec:gate_sensitivity}

The MCDM view treats activation as a decision rule, but a training algorithm needs a computable statistic.  We use two related quantities.  The first is the derivative magnitude $|\phi'(z)|$, which directly modulates the gate-projection gradient.  For an upstream signal $g$ of the same shape as $h$,
\begin{equation}
\label{eq:gate_grad_main}
\nabla_{W_g}\mathcal{L}=(g\odot u\odot \phi'(z))x^\top.
\end{equation}
Hence saturated or insensitive gate states produce low-utility gate updates, whereas responsive states make LoRA updates on the gate path effective.  For the commonly used SiLU gate $\phi(z)=z\sigma(z)$,
\begin{equation}
\label{eq:keff_silu_main}
|\phi'(z)|=\big|\sigma(z)+z\sigma(z)(1-\sigma(z))\big|.
\end{equation}

The second quantity describes the local scaling behavior of the channel.
\begin{definition}[Effective homogeneity profile]
\label{def:kappa_main}
For a differentiable self-gated nonlinearity $\phi$, define
\begin{equation}
\label{eq:kappa_def_main}
\kappa_\phi(z):=\frac{z\phi'(z)}{\phi(z)+\epsilon},
\qquad
\eh(z):=1+\kappa_\phi(z),
\end{equation}
where $\epsilon>0$ is used only for numerical stability.  We call $\eh(z)$ the effective homogeneity profile of the FFN channel.
\end{definition}

\begin{proposition}[Channel-wise local scaling sensitivity]
\label{prop:local_scaling_main}
Let $h_c(x)=\phi(z_c(x))u_c(x)$, where $z_c(x)=W_{g,c}x$ and $u_c(x)=W_{u,c}x$.  Along the scaling ray $x\mapsto \alpha x$, the logarithmic elasticity at $\alpha=1$ is
\begin{equation}
\left.\frac{\partial \log |h_c(\alpha x)|}{\partial \log \alpha}\right|_{\alpha=1}
=
1+\kappa_\phi(z_c(x))=\eh(z_c(x)).
\end{equation}
\end{proposition}

Proposition~\ref{prop:local_scaling_main} shows that the gate pre-activation controls the local scaling regime of each FFN channel.  The coarse regimes used by NA-LoRA are obtained from the characteristic points of the activation-specific effective-homogeneity curve, such as zeros, extrema, or discontinuities of $\kappa_\phi$.  For SwiGLU this yields suppressive, responsive, and positive non-amplifying regions; for piecewise-linear or other smooth gates the partition can collapse or become uninformative.  We therefore define the principle here and instantiate thresholds only through activation-specific characteristic points used in each experiment.

\begin{definition}[Effective-homogeneity movement]
\label{def:eh_constraint}
For consecutive LoRA states $\theta_t$ and $\theta_{t+1}$, define the data-dependent homogeneity movement
\begin{equation}
\label{eq:eh_consistency}
\mathcal{D}_{\mathrm{EH}}(t)
:=
\mathbb{E}_{x}\!\left[
\sum_{c=1}^{d_h} \omega_{t,c}(x)
\big(\eh(z_{t+1,c}(x))-\eh(z_{t,c}(x))\big)^2
\right],
\end{equation}
where $\omega_{t,c}(x)\ge0$ is a channel-importance weight, instantiated in NA-LoRA by the derivative-based mask.  We use $\mathcal{D}_{\mathrm{EH}}(t)$ as a local trajectory diagnostic and design guide, not as an auxiliary loss or as a requirement that cumulative raw gate movement decrease.
\end{definition}

\begin{proposition}[A local bound on homogeneity movement]
\label{prop:eh_drift_bound_main}
Assume $\kappa_\phi$ is $L_\kappa$-Lipschitz on the gate region visited by the minibatch.  Let $\delta z_{t,c}(x)=z_{t+1,c}(x)-z_{t,c}(x)$.  Then
\begin{equation}
\label{eq:eh_drift_bound_main}
\big|\eh(z_{t+1,c}(x))-\eh(z_{t,c}(x))\big|
\le L_\kappa |\delta z_{t,c}(x)|.
\end{equation}
If the LoRA-only update is scaled by $s_t$ and masked by $m_{t,c}$ on the gate output channel, then locally
\begin{equation}
\label{eq:eh_drift_mask_scale_bound}
\mathbb{E}_x\big[|\eh(z_{t+1,c})-\eh(z_{t,c})|^2\big]
\lesssim
\eta^2 s_t^2 m_{t,c}^2 L_\kappa^2\,
\mathbb{E}_x\big[\|G_{g,c,t}\|^2\|x\|^2\big],
\end{equation}
where $G_{g,c,t}$ denotes the effective gate-projection gradient for channel $c$ and $\eta$ is the learning rate.
\end{proposition}

This bound clarifies the two control variables used by NA-LoRA.  The mask controls \emph{where} homogeneity-sensitive movement is emphasized by downweighting low-responsiveness channels, and the step scale locally modulates \emph{how fast} FFN LoRA parameters move at each optimization step.  In practice we do not add Eq.~\eqref{eq:eh_consistency} as a loss; we implement its implication as lightweight training-time control on LoRA gradients.

\subsection{Temporal control as a design principle}
\label{subsec:temporal_stability_main}

Selection mismatch can accumulate over training.  Let $D_t(x)$ denote any scalar decision score induced by the FFN block at optimization step $t$; it may be a channel score, a logit contribution, or a coarse regime indicator.  We use the following elementary margin principle to justify temporal control.

\begin{proposition}[Margin-based flip bound]
\label{prop:margin_flip_main}
For any threshold $\tau>0$,
\begin{equation}
\label{eq:margin_flip_bound}
\mathbb{P}\!\left[\mathrm{sign}(D_{t+1}(x))\neq \mathrm{sign}(D_t(x))\right]
\le
\mathbb{P}\!\left[|D_t(x)|\le \tau\right]
+
\mathbb{P}\!\left[|D_{t+1}(x)-D_t(x)|>\tau\right].
\end{equation}
\end{proposition}

The first term captures inherently low-margin inputs.  The second term captures temporal score movement induced by the update.  NA-LoRA uses this term to motivate step scaling: the method shapes one-step movement while still allowing, and sometimes amplifying, task-driven movement on responsive channels.

\begin{corollary}[Step-size control of one-step movement]
\label{cor:step_drift_control}
Let $\theta_t$ denote the LoRA parameters affecting $D_t(x)$.  Assume $D_t(x)$ is locally Lipschitz in $\theta_t$ with constant $L_D(x)$ and the LoRA update has the form $\theta_{t+1}=\theta_t-\eta s_t g_t$.  Then, for any $\tau>0$,
\begin{equation}
\label{eq:step_drift_bound}
\mathbb{P}\!\big(|D_{t+1}(x)-D_t(x)|>\tau\big)
\le
\frac{\eta^2 s_t^2}{\tau^2}\,
\mathbb{E}\!\left[L_D(x)^2\|g_t\|^2\right].
\end{equation}
\end{corollary}

Together, the decomposition, effective-homogeneity profile, and temporal margin bound yield the core design rule of NA-LoRA: allocate rank to responsive selection channels and shape the speed and location of nonlinear gate movement under the low-rank budget.  Figure~\ref{fig:method} summarizes how this principle is implemented as a training-time controller.

\begin{figure*}[t]
    \centering
    \includegraphics[width=\textwidth]{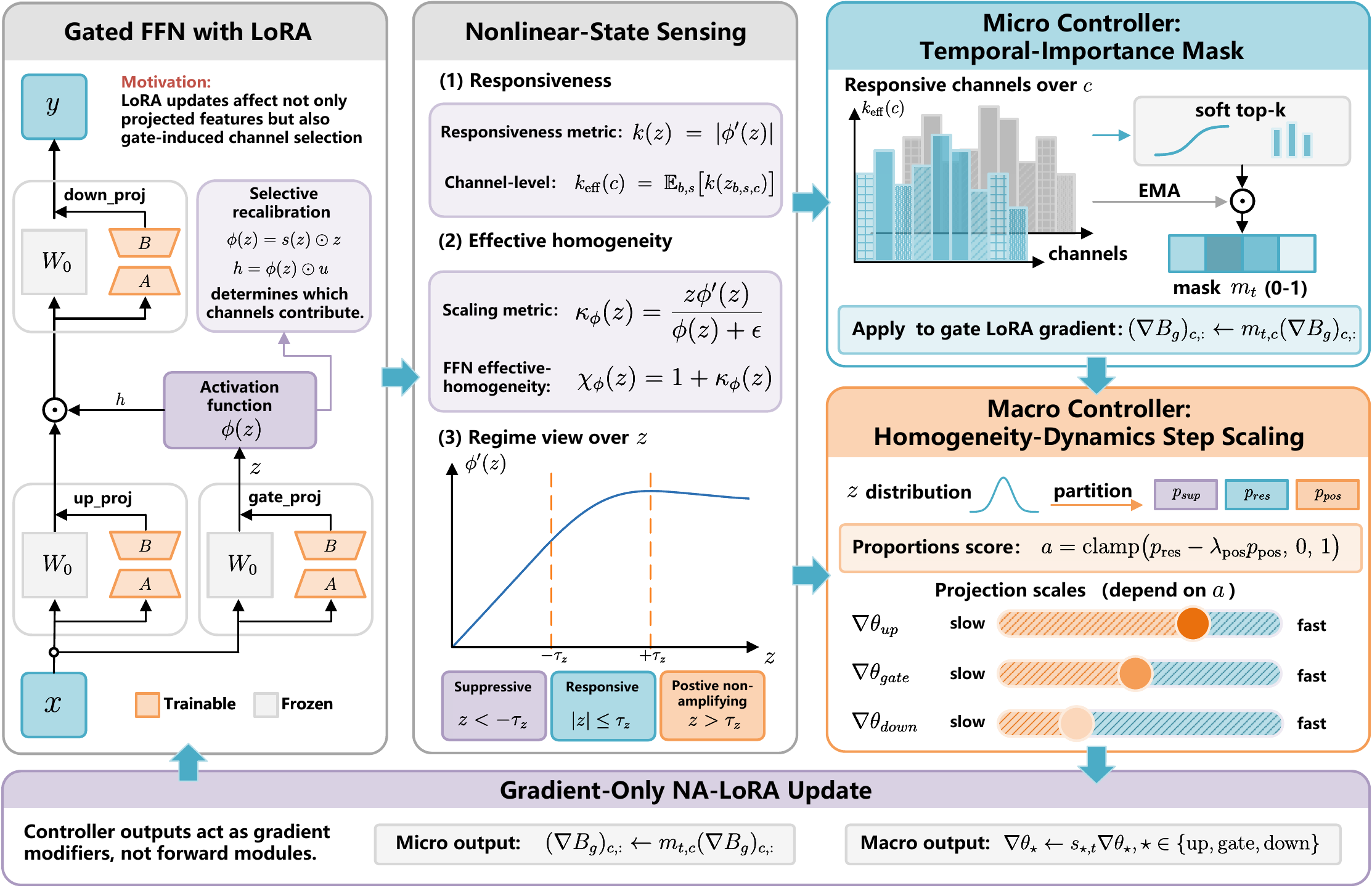}
    \caption{\textbf{NA-LoRA overview.} The temporal-importance mask allocates gate-related low-rank updates to responsive channels, while activation-specific step scaling modulates FFN LoRA update speed when a non-degenerate coarse regime statistic is available. Both controls are training-only.}
    \label{fig:method}
\end{figure*}

\section{Nonlinearity-Aware LoRA}
\label{sec:method}

\paragraph{What problem does NA-LoRA solve?}
Standard LoRA assigns low-rank updates according to modules, but not according to the nonlinear state of FFN channels.  In a gated FFN, however, the usefulness of updating a gate channel depends on whether that channel is locally insensitive, responsive, suppressive, or already positive and non-amplifying.  Under a tight rank budget, treating all channels uniformly can spend update capacity on directions that barely affect the activation-induced selection rule.  NA-LoRA addresses this \emph{nonlinearity-blind allocation} problem.  It complements gradient- and initialization-centric LoRA variants by asking where low-rank motion should be expressed inside the nonlinear FFN, rather than only how the low-rank update should approximate a full update.  Its goal is not to minimize raw sigmoid-gate movement, but to make LoRA-induced movement more concentrated on effective-homogeneity-sensitive channels and more structured over training.

\paragraph{Why should this be effective?}
The analysis in Section~\ref{sec:problem} gives two usable signals.  First, Eq.~\eqref{eq:gate_grad_main} shows that the actual gate-projection gradient is modulated by $\phi'(z)$; channels with larger $|\phi'(z)|$ have larger marginal ability to change the FFN utility through the gate path.  Second, the effective-homogeneity profile $\kappa_\phi(z)$ identifies activation-specific regions in which gate movement has different functional meaning.  These signals let NA-LoRA spend rank where gate changes are likely to matter and, when a non-degenerate coarse regime statistic exists, regulate FFN update speed according to the current nonlinear regime.  This interpretation is consistent with our diagnostics in Section~\ref{sec:exp}: NA-LoRA exhibits larger raw gate movement than LoRA, but that movement is accompanied by stronger and more sustained effective-homogeneity modulation, especially in middle and deeper layers.

NA-LoRA has two training-only components.  The \emph{temporal-importance channel mask} decides \emph{where} gate-related LoRA capacity should be used.  The \emph{homogeneity-dynamics step scaling} decides \emph{how fast} LoRA parameters in the FFN should move when the activation admits a meaningful coarse gate-regime statistic.  When the coarse statistic is unavailable or degenerate, the scaling is disabled and NA-LoRA reduces to sensitivity-based channel allocation.  Neither component changes the frozen backbone, adds an auxiliary loss, or introduces inference-time computation.

\subsection{Temporal-importance channel mask}
\label{subsec:temporal_importance_mask}

\paragraph{Channel allocation.}
Proposition~\ref{prop:sensitivity_weighted_mismatch} and Eq.~\eqref{eq:gate_grad_main} imply that gate-related updates are most useful when they affect channels that are both selected by the data and locally responsive to gate perturbations.  We use the derivative magnitude
\begin{equation}
\label{eq:keff_elem}
k(z):=|\phi'(z)|
\end{equation}
as a first-order proxy for this responsiveness.  For SiLU, this is
\begin{equation}
\label{eq:keff_silu}
k(z)=\big|\sigma(z)+z\sigma(z)(1-\sigma(z))\big|.
\end{equation}
Given gate pre-activations $z\in\mathbb{R}^{B\times S\times d_h}$ in one FFN block, the per-channel score is
\begin{equation}
\label{eq:kout_channel}
k_{\mathrm{eff}}(c):=\mathbb{E}_{b,s}\!\left[k(z_{b,s,c})\right],\qquad c\in[d_h].
\end{equation}
Large $k_{\mathrm{eff}}(c)$ indicates that channel $c$ repeatedly visits a state where gate-score perturbations can change the channel utility; small values indicate low marginal utility for gate-path adaptation.

\paragraph{Soft top-$k$ mask with temporal smoothing.}
Given a keep ratio $\rkeep\in(0,1]$, let
\begin{equation}
\tau_k=\mathrm{Quantile}\big(k_{\mathrm{eff}},1-\rkeep\big),
\end{equation}
and define
\begin{equation}
\label{eq:mask_soft}
m_c^{\mathrm{new}}=\sigma\!\big(\beta\,(k_{\mathrm{eff}}(c)-\tau_k)\big)\in(0,1),
\end{equation}
where $\beta>0$ controls boundary sharpness.  We smooth the mask over optimization steps,
\begin{equation}
\label{eq:mask_ema}
m_t=\eta_m m_{t-1}+(1-\eta_m)m^{\mathrm{new}},\qquad \eta_m\in[0,1),
\end{equation}
so that the selected support reflects persistent gate responsiveness rather than minibatch noise.

\paragraph{Gradient use.}
The mask is applied along the output-channel dimension of gate-related LoRA gradients.  With $B_g\in\mathbb{R}^{d_h\times r}$ and $A_g\in\mathbb{R}^{r\times d}$,
\begin{equation}
\label{eq:mask_apply_main}
(\nabla B_g)_{c,:}\leftarrow m_{t,c}(\nabla B_g)_{c,:}.
\end{equation}
Thus NA-LoRA changes only the adapter update.  Attention projections use standard LoRA because they do not contain the FFN gate-selection mechanism targeted here.

\subsection{Homogeneity-dynamics step scaling}
\label{subsec:homogeneity_dynamics_scaling}

\paragraph{Trajectory shaping.}
The mask reallocates gate-related LoRA updates across channels, but it does not control the overall speed of FFN LoRA updates.  Corollary~\ref{cor:step_drift_control} and Proposition~\ref{prop:eh_drift_bound_main} show that, locally, a LoRA-only scale controls an upper bound on one-step score and homogeneity change.  We use this fact as a trajectory-shaping device rather than as a claim that cumulative raw gate movement should be smaller than vanilla LoRA.  The scale can accelerate adaptation when a block has substantial responsive mass and can damp updates when the block is dominated by already positive non-amplifying states.  This component is used only when the activation provides a non-degenerate coarse effective-homogeneity partition.

\paragraph{Coarse gate-regime statistic.}
For SwiGLU FFN blocks, the characteristic points of the SiLU effective-homogeneity curve give a finite symmetric threshold $\tau_z=1.27846$.  We therefore partition gate pre-activations as
\begin{equation}
\label{eq:regime_probs}
p_{\mathrm{sup}} = \mathbb{E}\!\left[\mathbb{I}(z<-\tau_z)\right],\quad
p_{\mathrm{res}} = \mathbb{E}\!\left[\mathbb{I}(|z|\le \tau_z)\right],\quad
p_{\mathrm{pos}} = \mathbb{E}\!\left[\mathbb{I}(z>\tau_z)\right],
\end{equation}
where the expectation is over batch, token, and channel dimensions.  The three regions correspond to suppressive or negative gate-contribution states, responsive homogeneity-amplifying states, and positive non-amplifying states.

The threshold is activation-specific rather than a universal design constant.  For ReLU FFNs, such as \texttt{t5-base v1.0}, ReLU can be written as $\phi(z)=z\mathbb{I}(z>0)$ with a hard sign/Heaviside-type gate and a single characteristic point at $z=0$; the responsive band collapses, so $p_{\mathrm{res}}$ is not available.  For GELU FFNs, such as CLIP-ViT-B/16, the self-gate can be viewed as a smooth sigmoid-type gate without finite suppressive and positive boundaries under our coarse statistic; equivalently, the boundary is degenerate.  In both cases we disable homogeneity-dynamics step scaling and keep the sensitivity mask as the active NA-LoRA component.

We summarize the block state by
\begin{equation}
\label{eq:regime_score}
a=\mathrm{clamp}\big(p_{\mathrm{res}}-\lambda_{\mathrm{pos}}p_{\mathrm{pos}},\,0,\,1\big),
\end{equation}
where $\lambda_{\mathrm{pos}}>0$ penalizes dominance by the positive non-amplifying region.  A larger $a$ means that more mass lies in the responsive band, where additional LoRA motion is expected to be useful.

\paragraph{Projection-wise LoRA scaling.}
For each FFN projection $\star\in\{\mathrm{up},\mathrm{gate},\mathrm{down}\}$, we compute
\begin{equation}
\label{eq:step_scale}
s_\star^{\mathrm{new}}=
\mathrm{clip}\big(1+\alpha_\star(2a-1),\ s_{\min}^{\star},\ s_{\max}^{\star}\big),
\end{equation}
and optionally smooth it by
\begin{equation}
\label{eq:scale_ema}
s_{\star,t}=\beta_s s_{\star,t-1}+(1-\beta_s)s_\star^{\mathrm{new}}.
\end{equation}
The scale is applied only to the LoRA parameters of the corresponding projection:
\begin{equation}
\label{eq:lora_only_scale}
\nabla\theta_\star \leftarrow s_{\star,t}\nabla\theta_\star.
\end{equation}
The frozen backbone and forward computation remain unchanged.  This scaling should be interpreted as coarse block-level modulation: it does not prevent regime crossings and does not define instability by crossing frequency alone.  Rather, it works with the channel mask to concentrate adaptation on useful nonlinear states while avoiding unstructured amplification of the entire FFN adapter.

\paragraph{Summary.}
NA-LoRA solves a rank-allocation problem, and when the activation supports it, a trajectory-shaping problem in FFNs.  The mask increases the share of low-rank update capacity assigned to responsive gate channels, and the step scale adjusts FFN LoRA update speed using activation-specific effective-homogeneity regimes.  Together they promote structured nonlinear gate adaptation while preserving the simplicity and inference cost of standard LoRA.  The complete training-time procedure and hook-level details are provided in Appendix~\ref{app:impl}.

\section{Experiments}
\label{sec:exp}

\paragraph{Setup and naming.}
Following the experimental taxonomy of LoRA-GA and LoRA-Pro~\cite{wang2024loraga,wang2024lorapro}, we evaluate NA-LoRA on natural language understanding, large-language-model adaptation, and image classification. The main text focuses on Llama-3.1-8B-Base~\cite{dubey2024llama} and Llama-2-7B-Base~\cite{touvron2023llama} instruction tuning, including dialogue generation, mathematical reasoning, and code generation with WizardLM~\cite{xu2024wizardlm}, MetaMathQA~\cite{yu2024metamath}, and CodeFeedback~\cite{zheng2024opencodeinterpreter}. We evaluate these tasks with MT-Bench~\cite{zheng2024judging}, GSM8K~\cite{cobbe2021training}, and HumanEval~\cite{chen2021evaluating}. T5-Base/GLUE and CLIP-ViT-B/16 transfer results are reported in Appendix~\ref{app:add_exp}. Full details and hyperparameters are in Appendix~\ref{app:exp_details}. Unless stated otherwise, rank is $r=8$, all non-head linear modules are adapted, and NA-LoRA applies nonlinearity-aware controls only to FFN projections.

\paragraph{Baselines.}
We compare with full fine-tuning and representative PEFT baselines: LoRA~\cite{hu2022lora}, rsLoRA~\cite{kalajdzievski2023rank}, AdaLoRA~\cite{zhang2023adaptive}, DoRA~\cite{liu2024dora}, LoRA+~\cite{hayou2024lora+}, PiSSA~\cite{meng2024pissa}, LoRA-GA~\cite{wang2024loraga}, LoRA-Pro~\cite{wang2024lorapro}, and GoRA~\cite{heGoRAGradientdrivenAdaptive2025}.

\begin{table*}[t]
  \caption{\textbf{Large-language-model results.} Dialogue generation, mathematical reasoning, and code generation are evaluated by MT-Bench, GSM8K, and HumanEval, respectively.}
  \label{tab:results_large_combined}
  \centering
  \begin{small}
  \setlength{\tabcolsep}{3.7pt}
  \begin{tabular}{lcccccc}
    \toprule
    & \multicolumn{3}{c}{\textbf{Llama-3.1-8B-Base}} & \multicolumn{3}{c}{\textbf{Llama-2-7B-Base}} \\
    \cmidrule(lr){2-4} \cmidrule(lr){5-7}
    Method
    & \textbf{MT-Bench} & \textbf{GSM8K} & \textbf{HumanEval}
    & \textbf{MT-Bench} & \textbf{GSM8K} & \textbf{HumanEval} \\
    \midrule
    Full
    & 5.88$\pm$0.23 & 73.69$\pm$0.28 & 51.63$\pm$1.27
    & 5.30$\pm$0.11 & 59.36$\pm$0.85 & 35.31$\pm$2.13 \\
    \midrule
    LoRA
    & 6.08$\pm$0.02 & 71.52$\pm$0.73 & 41.05$\pm$2.85
    & 5.61$\pm$0.10 & 42.08$\pm$0.04 & 14.76$\pm$0.17 \\
    rsLoRA
    & 6.16$\pm$0.02 & 73.85$\pm$0.23 & 42.27$\pm$3.46
    & 5.25$\pm$0.03 & 45.62$\pm$0.10 & 16.01$\pm$0.79 \\
    AdaLoRA
    & 6.18$\pm$0.16 & 73.74$\pm$0.71 & 43.49$\pm$1.63
    & 5.57$\pm$0.05 & 50.72$\pm$1.39 & 17.80$\pm$0.44 \\
    DoRA
    & \underline{6.32$\pm$0.10} & 74.65$\pm$1.19 & 45.12$\pm$2.44
    & \textbf{5.97$\pm$0.02} & 53.07$\pm$0.75 & 19.75$\pm$0.41 \\
    LoRA+
    & 6.27$\pm$0.02 & 74.42$\pm$0.88 & 43.09$\pm$2.85
    & 5.71$\pm$0.08 & 52.11$\pm$0.62 & 18.17$\pm$0.52 \\
    PiSSA
    & 6.10$\pm$0.01 & 73.59$\pm$1.11 & 43.70$\pm$4.47
    & 5.30$\pm$0.02 & 44.54$\pm$0.27 & 16.02$\pm$0.17 \\
    LoRA-GA
    & 6.00$\pm$0.05 & 74.30$\pm$1.06 & 44.10$\pm$2.85
    & \underline{5.95$\pm$0.16} & 53.60$\pm$0.30 & 19.81$\pm$1.46 \\
    LoRA-Pro
    & 6.21$\pm$0.06 & \underline{74.93$\pm$1.01} & 43.63$\pm$1.49
    & 5.72$\pm$0.03 & \underline{57.57$\pm$0.50} & 22.97$\pm$0.35 \\
    GoRA
    & 6.30$\pm$0.02 & 74.55$\pm$1.04 & \underline{45.32$\pm$1.62}
    & 5.61$\pm$0.12 & 54.04$\pm$0.22 & \underline{24.80$\pm$1.04} \\
    \midrule
    NA-LoRA$_{r=8}$
    & \textbf{6.39$\pm$0.04} & \textbf{75.16$\pm$1.34} & \textbf{46.34$\pm$1.22}
    & 5.92$\pm$0.03 & \textbf{57.97$\pm$3.08} & \textbf{26.02$\pm$0.81} \\
    NA-LoRA$_{r=32}$
    & 6.20$\pm$0.02 & 75.33$\pm$1.42 & 47.13$\pm$1.01
    & 5.72$\pm$0.04 & 59.46$\pm$0.78 & 27.64$\pm$0.41 \\
    NA-LoRA$_{r=128}$
    & 6.12$\pm$0.04 & 76.06$\pm$0.48 & 47.76$\pm$1.02
    & 5.78$\pm$0.02 & 59.74$\pm$1.13 & 28.66$\pm$2.44 \\
    \bottomrule
  \end{tabular}
  \par\vspace{4pt}
  \begin{minipage}{0.92\textwidth}
  \footnotesize\raggedright All PEFT methods adapt non-head linear layers under matched protocols (details in Appendix~\ref{app:exp_details}); NA-LoRA's nonlinearity-aware controls apply only to FFN projections.
  \end{minipage}
  \end{small}
\end{table*}

\subsection{Results on Large Language Models}
\label{subsec:llm_results}

\paragraph{Main results.}
Table~\ref{tab:results_large_combined} summarizes the LLM results. Under the matched rank-$8$ setting, NA-LoRA obtains the best PEFT results on all three Llama-3.1-8B-Base tasks, improving vanilla LoRA by $0.31$, $3.64$, and $5.29$ points on MT-Bench, GSM8K, and HumanEval. It also outperforms strong LoRA variants that improve scaling, initialization, rank allocation, or gradient alignment, suggesting that FFN nonlinear selection provides complementary information to existing weight- or gradient-space designs.

On Llama-2-7B-Base, NA-LoRA achieves the best rank-$8$ GSM8K and HumanEval scores while remaining close to the best MT-Bench result. The gains are more pronounced on reasoning and code-generation tasks, where FFN-mediated feature recombination is likely more important. Overall, NA-LoRA consistently improves the low-rank frontier without changing inference-time computation.

\noindent\textbf{Rank scaling.}
Larger ranks further improve GSM8K and HumanEval in several cases, but the trend is not monotonic for MT-Bench. This indicates that NA-LoRA is not simply a capacity-increasing method: increasing rank also changes the strength of gate-selection modulation, which may benefit reasoning and code generation but does not necessarily improve open-ended dialogue evaluation. We therefore use $r=8$ for the main fair comparison.

\noindent\textbf{FFN-specific evidence.}
Since NA-LoRA modifies only FFN LoRA updates, we further examine whether the gains are tied to FFN adaptation rather than generic module-wise capacity. Appendix~\ref{app:add_exp} shows that attention-only LoRA substantially underperforms all-linear LoRA on GSM8K, whereas FFN-only LoRA recovers most of the all-linear performance and remains competitive on HumanEval. Under the same FFN-only target-module setting, NA-LoRA further improves over LoRA. This supports the central design choice of applying nonlinearity-aware control to FFN projections, where gate activations directly determine channel selection.

\paragraph{Gate-selection and effective-homogeneity dynamics.}
To examine whether NA-LoRA changes the intended nonlinear behavior, we record FFN gate pre-activations from \texttt{mlp.gate\_proj} during MetaMathQA fine-tuning of Llama-3.1-8B. Diagnostics are computed on a fixed GSM8K-train probe subset by comparing adjacent checkpoints saved every 250 steps. The cumulative raw sigmoid-gate movement $D_{\mathrm{sel}}(t)$ is larger for NA-LoRA than LoRA (0.3744 vs. 0.3199), so NA-LoRA should not be interpreted as suppressing gate movement. Instead, its effect is to make gate movement more concentrated in behaviorally meaningful nonlinear regimes.

For SiLU, we compute $\kappa_\phi(z)=z\phi'(z)/(\phi(z)+10^{-6})$ and partition gate scores with $\tau_z=1.27846$. Figures~\ref{fig:effective_homogeneity_diagnostics}--\ref{fig:effective_homogeneity_layerwise} show that NA-LoRA yields larger cumulative regime flips (0.1609 vs. 0.1421), larger $\kappa_\phi$ drift (1.0896 vs. 0.9349), and larger suppressive-regime crossings (0.1449 vs. 0.1279). These results support the proposed mechanism: NA-LoRA induces stronger effective-homogeneity modulation rather than merely smaller raw displacement. Meanwhile, severe regime jumps remain nearly zero for both methods ($1.10\times10^{-7}$ for LoRA and $3.36\times10^{-7}$ for NA-LoRA), indicating that the stronger modulation is structured rather than unstable.

\noindent\textbf{Target modules and transfer.}
Additional T5/GLUE and CLIP-ViT-B/16 results in Appendix~\ref{app:add_exp} further test activation-specific boundary cases where homogeneity step scaling is disabled, showing that the temporal-importance mask remains useful beyond SwiGLU FFNs.

\begin{figure}[t]
    \centering
    \includegraphics[width=\linewidth]{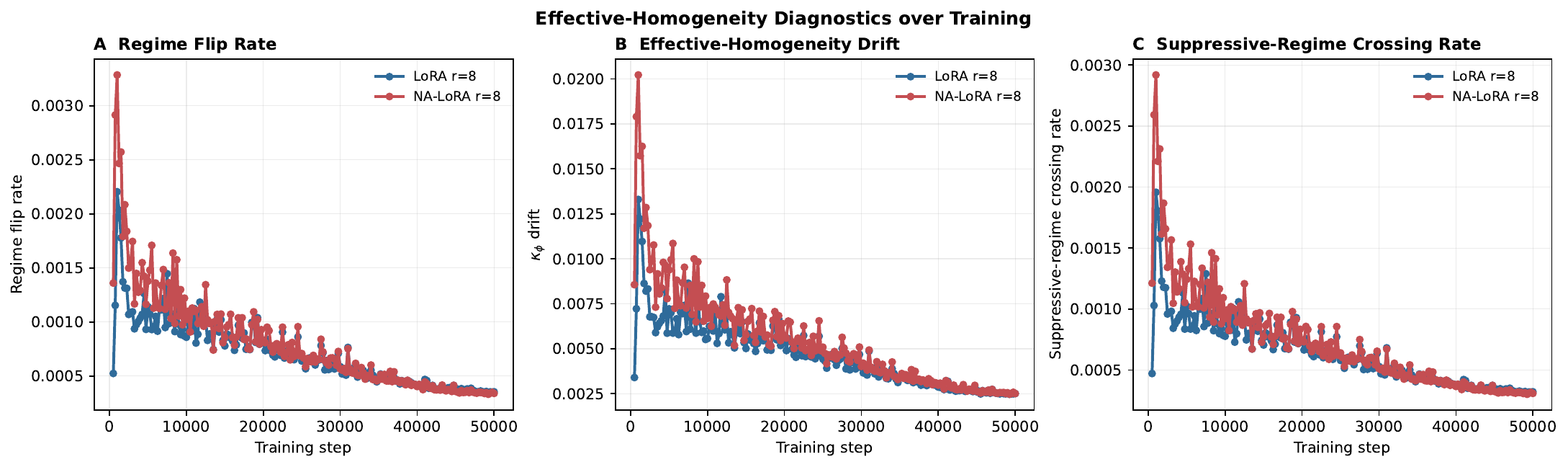}
    \caption{\textbf{Effective-homogeneity diagnostics over training.} Adjacent-checkpoint regime flip rate, $\kappa_\phi$ drift, and suppressive-regime crossing rate on the fixed GSM8K-train probe subset.}
    \label{fig:effective_homogeneity_diagnostics}
\end{figure}

\begin{figure}[t]
    \centering
    \includegraphics[width=\linewidth]{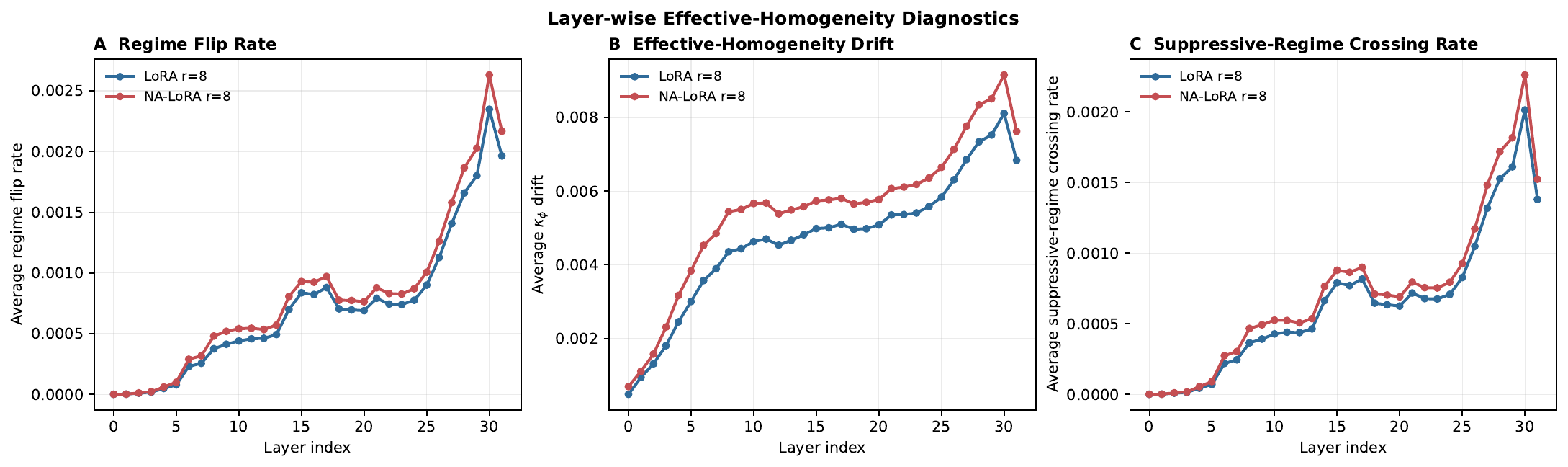}
    \caption{\textbf{Layer-wise effective-homogeneity diagnostics.} Layer averages of the three diagnostics in Figure~\ref{fig:effective_homogeneity_diagnostics}.}
    \label{fig:effective_homogeneity_layerwise}
\end{figure}

\begin{table}[H]
  \caption{\textbf{Ablation on Llama-3.1-8B-Base ($r{=}8$).}}
  \label{tab:ablation3}
  \centering
  \begin{small}
    \begin{tabular}{lcc}
      \toprule
      Method & \textbf{GSM8K} & \textbf{HumanEval} \\
      \midrule
      FT & 73.69$\pm$0.28 & 51.63$\pm$1.27 \\
      LoRA & 71.52$\pm$0.73 & 41.05$\pm$2.85 \\
      rsLoRA & 73.85$\pm$0.23 & 42.27$\pm$3.46 \\
      NA-LoRA & \textbf{75.16$\pm$1.34} & \textbf{46.34$\pm$1.22} \\
      NA-LoRA \textnormal{w/o} channel mask & 74.53$\pm$1.82 & \underline{45.53$\pm$1.42} \\
      NA-LoRA \textnormal{w/o} scaling step & 74.43$\pm$0.63 & 44.31$\pm$5.69 \\
      \bottomrule
    \end{tabular}
  \end{small}
\end{table}

\paragraph{Ablation study.}
Table~\ref{tab:ablation3} isolates the two NA-LoRA components on Llama-3.1-8B-Base with $r=8$. Removing either the temporal-importance mask or the homogeneity-dynamics step scaling weakens performance on at least one metric. This suggests that the two controls play complementary roles: the mask determines where gate-related low-rank updates should be expressed, while the scaling step controls how fast FFN LoRA parameters move through nonlinear regimes. The larger HumanEval variance without scaling further suggests that trajectory control helps stabilize code-generation adaptation.

\section{Conclusion}
\label{sec:conclusion}

We revisited the LoRA--full-fine-tuning gap through an evaluation--decision view of self-gated FFNs.  The key issue is that update-space approximation alone does not ensure behavioral alignment: low-rank residuals can alter the activation-induced score-to-weight recalibration that decides channel usage.  NA-LoRA addresses this selection misalignment with two training-only controls: a derivative-based temporal mask and, when the activation permits, a regime-aware homogeneity-dynamics scale, while leaving inference identical to standard LoRA.  Experiments and gate-dynamics diagnostics show that behavior-aware, effective-homogeneity-sensitive modulation is a useful complement to weight-, gradient-, and rank-centric PEFT design.

\bibliographystyle{iclr2025_conference}
\bibliography{ref}

\appendix
\onecolumn

\section{Experimental Details}
\label{app:exp_details}

\paragraph{Data and models.}
We follow the naming and setup of LoRA-GA/LoRA-Pro-style evaluations~\cite{wang2024lorapro,wang2024loraga,heGoRAGradientdrivenAdaptive2025}: natural language understanding uses T5-Base~\cite{raffel2020exploring} on GLUE~\cite{wang2019glue}, large-language-model experiments include dialogue generation, mathematical reasoning, and code generation on Llama models, and image classification uses CLIP-ViT-B/16 transfer~\cite{radford2021learning}.  The main Llama experiments use 100K MetaMathQA examples for mathematical reasoning~\cite{yu2024metamath}, 100K CodeFeedback examples for code generation~\cite{zheng2024opencodeinterpreter}, and a 52K WizardLM subset for dialogue generation~\cite{xu2024wizardlm}.  The main backbones are Llama-3.1-8B-Base~\cite{dubey2024llama} and Llama-2-7B-Base~\cite{touvron2023llama}.

\paragraph{Adapters and optimization.}
Unless otherwise stated, all PEFT methods adapt all linear modules except embeddings, normalization layers, and output heads.  For attention projections, NA-LoRA uses the vanilla LoRA update; the temporal-importance mask and homogeneity-dynamics scaling are applied only to FFN projections.  When the activation-specific coarse regime statistic degenerates, as in ReLU-FFN T5 v1.0 or GELU-based CLIP-ViT, homogeneity-dynamics scaling is disabled.  We use AdamW~\cite{loshchilov2019decoupled} and match the learning-rate schedule, batch size, sequence length, and epochs of the LoRA-Pro protocol~\cite{wang2024lorapro}, which follows LoRA-GA-style settings~\cite{wang2024loraga}.  Unless otherwise stated, NA-LoRA uses learning rate $2\times10^{-5}$.

\paragraph{Evaluation and reporting.}
We follow prior work for MT-Bench judging~\cite{zheng2024judging}, GSM8K accuracy~\cite{cobbe2021training}, HumanEval pass@1~\cite{chen2021evaluating}, and mean$\pm$std reporting over seeds; MT-Bench responses are judged using GPT-4~\cite{achiam2023gpt}.  All main experiments are run on $4\times$A100 GPUs.  The additional GLUE and CLIP transfer settings are described with their corresponding tables below.

\section{Additional Experimental Results}
\label{app:add_exp}

\paragraph{Effect of target modules.}
Table~\ref{tab:target_modules} varies the adapted module group on Llama-3.1-8B-Base at rank $r=8$.  Attention-only LoRA degrades GSM8K relative to all-linear LoRA, while FFN-only LoRA recovers most of the math performance and remains competitive on code.  Replacing FFN-only LoRA with NA-LoRA further improves both GSM8K and HumanEval under the same target-module setting, supporting the claim that nonlinearity-aware modulation is most useful where FFN gates directly determine selection dynamics.

\begin{table}[H]
  \caption{\textbf{Target-module ablation on Llama-3.1-8B-Base ($r{=}8$).}}
  \label{tab:target_modules}
  \centering
  \begin{small}
    \begin{tabular}{llcc}
      \toprule
      Method & Target Modules & \textbf{GSM8K} & \textbf{HumanEval} \\
      \midrule
      LoRA & All linear & 71.52$\pm$0.73 & 41.05$\pm$2.85 \\
      LoRA & Attention only & 67.78$\pm$1.25 & \underline{43.09$\pm$0.35} \\
      LoRA & FFN only & \underline{71.97$\pm$1.23} & 42.07$\pm$0.61 \\
      NA-LoRA & FFN only & \textbf{73.34$\pm$0.63} & \textbf{44.31$\pm$2.64} \\
      \bottomrule
    \end{tabular}
  \end{small}
\end{table}

\subsection{Results on Natural Language Understanding Tasks}\label{subsec:t5_glue}
\paragraph{T5-Base on GLUE.}
Table~\ref{tab:t5} reports five GLUE sub-tasks~\cite{wang2019glue} using T5-Base~\cite{raffel2020exploring}.  Prior LoRA-GA/LoRA-Pro/GoRA evaluations use \texttt{t5-base v1.0}, whose FFN uses ReLU rather than the SwiGLU gates used in the main Llama analysis; gated GeGLU is adopted in \texttt{t5-base v1.1}.  Under the self-gate rewriting $\mathrm{ReLU}(z)=z\mathbb{I}(z>0)$, the gate is a hard sign/Heaviside function with only the characteristic point $z=0$.  The responsive interval therefore collapses and there is no coarse $p_{\mathrm{res}}$ statistic.  For a protocol-compatible comparison, \textbf{NA-LoRA*} disables homogeneity-dynamics step scaling and applies only the temporal-importance mask.  The strong results in this activation-specific boundary case suggest that sensitivity-based channel allocation remains useful even when the coarse regime statistic is unavailable.

\begin{table}[H]
  \caption{\textbf{T5-Base on GLUE.} NA-LoRA* is mask-only because ReLU FFNs do not provide a non-degenerate coarse-regime statistic.}
  \label{tab:t5}
  \centering
  \begin{small}
  \begin{tabular}{lcccccc}
    \toprule
    Method & \textbf{MNLI} & \textbf{SST-2} & \textbf{CoLA} & \textbf{QNLI} & \textbf{MRPC} & \textbf{Average} \\
    \midrule
    Full & 86.33$\pm$0.00 & 94.75$\pm$0.21 & 80.70$\pm$0.24 & 93.19$\pm$0.22 & 84.56$\pm$0.73 & 87.91 \\
    \midrule
    LoRA & 85.30$\pm$0.04 & 94.04$\pm$0.11 & 69.35$\pm$0.05 & 92.96$\pm$0.09 & 68.38$\pm$0.01 & 82.08 \\
    PiSSA & 85.75$\pm$0.07 & 94.07$\pm$0.06 & 74.27$\pm$0.39 & 93.15$\pm$0.14 & 76.31$\pm$0.51 & 84.71 \\
    rsLoRA & 85.73$\pm$0.10 & \underline{94.19$\pm$0.23} & 72.32$\pm$1.12 & 93.12$\pm$0.09 & 52.86$\pm$2.27 & 79.64 \\
    LoRA+ & 85.81$\pm$0.09 & 93.85$\pm$0.24 & 77.53$\pm$0.20 & 93.14$\pm$0.03 & 74.43$\pm$1.39 & 84.95 \\
    DoRA & 85.67$\pm$0.09 & 94.04$\pm$0.53 & 72.04$\pm$0.94 & 93.04$\pm$0.06 & 68.08$\pm$0.51 & 82.57 \\
    AdaLoRA & 85.45$\pm$0.11 & 93.69$\pm$0.20 & 69.16$\pm$0.24 & 91.66$\pm$0.05 & 68.14$\pm$0.28 & 81.62 \\
    LoRA-GA & 85.70$\pm$0.09 & 94.11$\pm$0.18 & 80.57$\pm$0.20 & 93.18$\pm$0.06 & 85.29$\pm$0.24 & 87.77 \\
    LoRA-Pro & \underline{86.03$\pm$0.19} & \underline{94.19$\pm$0.13} & \textbf{81.94$\pm$0.24}
    & \underline{93.42$\pm$0.05} & \underline{86.60$\pm$0.14} & \textbf{88.44} \\
    GoRA & 85.91$\pm$0.02 & \textbf{94.68$\pm$0.43} & 79.86$\pm$0.35 & 93.27$\pm$0.08 & 86.10$\pm$0.20 & 87.96 \\
    NA-LoRA* & \textbf{86.29$\pm$0.10} & 93.96$\pm$0.07 & \underline{81.46$\pm$0.13} & \textbf{93.44$\pm$0.06} & \textbf{87.01$\pm$0.00} & \underline{88.43} \\
    \bottomrule
  \end{tabular}
  \end{small}

\end{table}

\subsection{Results on Image Classification Tasks}\label{subsec:clip_transfer_results}
\paragraph{CLIP-ViT-B/16 transfer.}
Table~\ref{tab:clip_transfer} follows the LoRA-Pro image-classification protocol~\cite{wang2024lorapro} with CLIP-ViT-B/16~\cite{radford2021learning} transfer on Cars~\cite{krause20133d}, DTD~\cite{cimpoi2014describing}, EuroSAT~\cite{helber2019eurosat}, GTSRB~\cite{houben2013detection}, RESISC45~\cite{cheng2017remote}, SUN397~\cite{xiao2010sun}, and SVHN~\cite{netzer2011reading}.  CLIP-ViT-B/16 uses GELU FFNs; under a self-gate rewriting, GELU behaves as a smooth sigmoid-type gate and does not provide finite suppressive and positive boundaries for the coarse statistic used in Eq.~\eqref{eq:regime_probs}.  We therefore disable homogeneity-dynamics step scaling for CLIP-ViT and use the temporal-importance mask as the active NA-LoRA component.  NA-LoRA achieves the best average accuracy, indicating that the nonlinearity-aware allocation principle remains effective in this activation-specific boundary case rather than requiring all activations to share the same regime structure.

\begin{table}[H]
  \caption{\textbf{CLIP-ViT-B/16 transfer.} Accuracy on seven image-classification datasets under the LoRA-Pro protocol.}
  \label{tab:clip_transfer}
  \centering
  \begin{small}
  \resizebox{\textwidth}{!}{%
  \begin{tabular}{lcccccccc}
    \toprule
    Method & \textbf{Cars} & \textbf{DTD} & \textbf{EuroSAT} & \textbf{GTSRB} & \textbf{RESISC45} & \textbf{SUN397} & \textbf{SVHN} & \textbf{Average} \\
    \midrule
    Zero-shot & 63.75 & 44.39 & 42.22 & 35.22 & 56.46 & 62.56 & 15.53 & 45.73 \\
    Full FT & 84.23$\pm$0.06 & 77.44$\pm$0.19 & 98.09$\pm$0.03 & 94.31$\pm$0.28 & 93.95$\pm$0.00 & 75.35$\pm$0.10 & 93.04$\pm$0.18 & 88.06 \\
    \midrule
    LoRA & 72.81$\pm$0.13 & 73.92$\pm$0.38 & 96.93$\pm$0.07 & 92.40$\pm$0.10 & 90.03$\pm$0.14 & 70.12$\pm$0.18 & 88.02$\pm$0.07 & 83.46 \\
    rsLoRA & 82.38$\pm$0.20 & 78.03$\pm$0.76 & 98.06$\pm$0.08 & 95.04$\pm$0.11 & 93.96$\pm$0.18 & 75.38$\pm$0.24 & 92.74$\pm$0.18 & 87.94 \\
    LoRA+ & 72.87$\pm$0.18 & 74.07$\pm$0.45 & 97.01$\pm$0.02 & 92.42$\pm$0.18 & 89.96$\pm$0.11 & 70.17$\pm$0.15 & 88.08$\pm$0.05 & 83.51 \\
    DoRA & 73.72$\pm$0.06 & 73.72$\pm$0.33 & 96.95$\pm$0.01 & 92.38$\pm$0.17 & 90.03$\pm$0.08 & 70.20$\pm$0.19 & 88.23$\pm$0.05 & 83.48 \\
    LoRA-GA & 85.18$\pm$0.41 & 77.50$\pm$0.12 & 98.05$\pm$0.27 & 95.28$\pm$0.10 & 94.43$\pm$0.19 & 75.44$\pm$0.06 & 93.68$\pm$0.35 & 88.51 \\
    LoRA-Pro & \textbf{85.87$\pm$0.08} & \textbf{78.64$\pm$0.25} & 98.46$\pm$0.03 & 95.66$\pm$0.05 & 94.75$\pm$0.21 & \underline{76.42$\pm$0.14} & 94.63$\pm$0.20 & 89.20 \\
    GoRA & \underline{85.76$\pm$0.19} & 78.17$\pm$0.32 & \underline{98.77$\pm$0.35} & \textbf{96.66$\pm$0.36} & \underline{95.16$\pm$0.26} & \textbf{76.46$\pm$0.08} & \underline{95.32$\pm$0.13} & \underline{89.47} \\
    \midrule
    NA-LoRA & 85.28$\pm$0.20 & \underline{78.26$\pm$0.77} & \textbf{98.96$\pm$0.04} & \underline{96.14$\pm$0.46} & \textbf{95.57$\pm$0.11} & 75.84$\pm$0.16 & \textbf{97.14$\pm$0.04} & \textbf{89.60} \\
    \bottomrule
  \end{tabular}%
  }
  \end{small}
\end{table}

\section{Notation and Conceptual Mapping}
\label{app:notation}

This appendix summarizes the notation and formalizes the Multi-Criteria Decision Making (MCDM) mapping introduced in Section~\ref{subsec:selection_view}.

\subsection{Standard Notation}
We denote input representations by $x \in \mathbb{R}^d$. A standard Transformer FFN block (SwiGLU variant) is defined as:
\begin{equation}
    u = W_u x, \quad z = W_g x, \quad h = \phi(z) \odot u, \quad y = W_d h,
\end{equation}
where $\phi(\cdot)$ is a self-gated nonlinearity. We denote the LoRA rank by $r$ and write a low-rank update as $BA$, absorbing the usual LoRA scaling factor into $B$ or $A$ for notational simplicity.

\subsection{MCDM Mapping: View vs. Selection}
\label{app:mcdm_mapping}

As defined in Definition~\ref{def:selection_weights}, we decompose the nonlinearity $\phi$ into a selection weight $s(z)$ and the value $z$.  Table~\ref{tab:mcdm_mapping} summarizes the terminology used by the main text.
\begin{table}[H]
    \centering
    \caption{MCDM mapping used for the self-gated FFN notation.}
    \label{tab:mcdm_mapping}
    \begin{small}
    \begin{tabular}{l|l|l}
        \toprule
        \textbf{FFN Component} & \textbf{Symbol} & \textbf{MCDM Interpretation} \\
        \midrule
        Input State & $x$ & Context / Environment \\
        Linear Projection & $u = W_u x$ & \textbf{Candidate Generation} (The ``View") \\
        Gate Pre-activation & $z = W_g x$ & \textbf{Evaluation Score} (Criteria) \\
        Nonlinearity & $s_i(z)=\varrho(z_i)$ & \textbf{Decision Weighting} (Selection Policy) \\
        Output & $h = \phi(z)\odot u=(s(z)\odot z)\odot u$ & \textbf{Reweighted Utility} \\
        \bottomrule
    \end{tabular}
    \end{small}
\end{table}

This table fixes terminology used by the main text; the formal mismatch decomposition is proved next.

\section{Proofs: Exact Selection-Mismatch Decomposition}
\label{app:selection_drift}

This section proves Proposition~\ref{prop:exact_selection_decomp} and Corollary~\ref{cor:selection_amplification}.

\subsection{Proof of Proposition~\ref{prop:exact_selection_decomp}}
\begin{proof}
Let $s_{\mathrm{tar}}=s(z_{\mathrm{tar}})$ and $s_{\mathrm{lora}}=s(z_{\mathrm{lora}})$.
By definition,
\[
\mathcal{U}(z_{\mathrm{tar}},u_{\mathrm{tar}})
-\mathcal{U}(z_{\mathrm{lora}},u_{\mathrm{lora}})
=
\phi(z_{\mathrm{tar}})\odot u_{\mathrm{tar}}
-
\phi(z_{\mathrm{lora}})\odot u_{\mathrm{lora}}.
\]
Add and subtract $\phi(z_{\mathrm{tar}})\odot u_{\mathrm{lora}}$:
\[
=
\big(\phi(z_{\mathrm{tar}})-\phi(z_{\mathrm{lora}})\big)\odot u_{\mathrm{lora}}
+
\phi(z_{\mathrm{tar}})\odot (u_{\mathrm{tar}}-u_{\mathrm{lora}}).
\]
Using $\phi(z)=s(z)\odot z$,
\[
\phi(z_{\mathrm{tar}})-\phi(z_{\mathrm{lora}})
=
s_{\mathrm{tar}}\odot z_{\mathrm{tar}} - s_{\mathrm{lora}}\odot z_{\mathrm{lora}}
=
s_{\mathrm{tar}}\odot(z_{\mathrm{tar}}-z_{\mathrm{lora}})
+
(s_{\mathrm{tar}}-s_{\mathrm{lora}})\odot z_{\mathrm{lora}}.
\]
Substituting $\Delta z=z_{\mathrm{tar}}-z_{\mathrm{lora}}$ and $\Delta u=u_{\mathrm{tar}}-u_{\mathrm{lora}}$ yields Eq.~\eqref{eq:exact_selection_decomp}.
\end{proof}

\subsection{Proof of Corollary~\ref{cor:selection_amplification}}
\begin{proof}
Apply triangle inequality to Eq.~\eqref{eq:exact_selection_decomp} and bound each Hadamard product by the corresponding infinity norm factor.
\end{proof}
\subsection{Proof of Proposition~\ref{prop:sensitivity_weighted_mismatch}}
\begin{proof}
For each channel $c$, apply the mean-value theorem to the scalar function $s$ on the segment between $z_{\mathrm{lora},c}$ and $z_{\mathrm{tar},c}$.
There exists $\bar z_c$ on this segment such that
\[
s(z_{\mathrm{tar},c})-s(z_{\mathrm{lora},c})=s'(\bar z_c)(z_{\mathrm{tar},c}-z_{\mathrm{lora},c})=s'(\bar z_c)\Delta z_c.
\]
Multiplying both sides by $z_{\mathrm{lora},c}u_{\mathrm{lora},c}$ yields Eq.~\eqref{eq:sensitivity_weighted_mismatch}; squaring and summing over channels yields Eq.~\eqref{eq:sensitivity_weighted_residual}.
\end{proof}

\subsection{Why the bound is data-aware}
\label{proof:weight_vs_selection}

Corollary~\ref{cor:selection_amplification} should not be read as claiming that every small matrix residual causes a large functional error.  If inputs, utilities, and gate Lipschitz constants are uniformly bounded, then sufficiently small residuals also control selection change.  The point is that any such statement is necessarily \emph{data-aware}: it depends on the directions $E_gx$, the local sensitivity of $s$, and the magnitudes of $z_{\mathrm{lora}}$ and $u_{\mathrm{lora}}$.  Standard weight-space objectives control only $E_g$ and therefore miss the channel-wise weighting appearing in Eq.~\eqref{eq:sensitivity_weighted_residual}.  NA-LoRA is designed to estimate this missing weighting from the observed gate dynamics.

\section{Local Scaling Sensitivity and Practical Proxy}
\label{app:local_homogeneity}

This section provides the proofs and implementation-level justification for the effective-homogeneity and sensitivity statistics used by NA-LoRA.

\subsection{Proof of Proposition~\ref{prop:local_scaling_main}}
\label{proof:homogeneity}

\begin{proof}
Fix a channel and abbreviate $z=z_c(x)$ and $u=u_c(x)$.  Along the scaling ray $x\mapsto \alpha x$, the linear projections satisfy $z_c(\alpha x)=\alpha z$ and $u_c(\alpha x)=\alpha u$.  Define
\[
g(\alpha)=h_c(\alpha x)=\phi(\alpha z)(\alpha u).
\]
The logarithmic elasticity of $g$ at $\alpha=1$ is
\[
\left.\frac{\partial \log |g(\alpha)|}{\partial \log \alpha}\right|_{\alpha=1}
=
\left.\frac{\alpha}{g(\alpha)}\frac{\partial g(\alpha)}{\partial \alpha}\right|_{\alpha=1}.
\]
Using the chain rule,
\[
\frac{\partial g(\alpha)}{\partial \alpha}
= z\phi'(\alpha z)(\alpha u)+\phi(\alpha z)u.
\]
Evaluating at $\alpha=1$ gives
\[
\frac{1}{\phi(z)u}\big(z\phi'(z)u+\phi(z)u\big)
=1+\frac{z\phi'(z)}{\phi(z)}.
\]
With the numerical regularizer in Definition~\ref{def:kappa_main}, this is $1+\kappa_\phi(z)=\eh(z)$.
\end{proof}

\subsection{Activation-specific characteristic points}
The coarse regime statistic is derived from the activation's effective-homogeneity curve rather than assigned as a universal constant.  For SiLU, $\phi(z)=z\sigma(z)$ and, away from zero,
\begin{equation}
\label{eq:silu_kappa_appendix}
\kappa_\phi(z)=\frac{z\phi'(z)}{\phi(z)}=1+z(1-\sigma(z)).
\end{equation}
Solving $\kappa_\phi(z)=0$ gives $z\approx-1.27846$.  The derivative of Eq.~\eqref{eq:silu_kappa_appendix} is $(1-\sigma(z))(1-z\sigma(z))$, so the positive stationary point satisfies $z\sigma(z)=1$, yielding $z\approx1.27846$.  These two characteristic points motivate the symmetric SwiGLU threshold $\tau_z=1.27846$ used in the Llama regime statistic.

For ReLU, $\phi(z)=z\mathbb{I}(z>0)$ has a hard gate with a single discontinuity at $z=0$; the responsive interval collapses, so the three-way statistic in Eq.~\eqref{eq:regime_probs} is not used.  For GELU in CLIP-ViT-B/16, the self-gate is smooth and sigmoid-like, and our coarse suppressive/positive boundaries are not finite.  In both boundary cases, the experiments disable homogeneity-dynamics step scaling and use the temporal-importance mask only.  These cases test whether nonlinearity-aware allocation remains useful without assuming a shared SiLU-style regime structure.

\subsection{Proof of Proposition~\ref{prop:eh_drift_bound_main}}
\begin{proof}
Since $\eh(z)=1+\kappa_\phi(z)$, the assumed $L_\kappa$-Lipschitzness of $\kappa_\phi$ implies
\[
|\eh(z_{t+1,c})-\eh(z_{t,c})|
=|\kappa_\phi(z_{t+1,c})-\kappa_\phi(z_{t,c})|
\le L_\kappa |z_{t+1,c}-z_{t,c}|,
\]
which proves Eq.~\eqref{eq:eh_drift_bound_main}.

For the second statement, locally linearize the change in the gate projection under one LoRA optimization step.  Let $G_{g,c,t}$ denote the effective channel-wise gradient direction after accounting for the low-rank parameterization.  If the update is masked and scaled, then the induced first-order change in the gate score has the form
\[
\delta z_{t,c}(x)\approx -\eta s_t m_{t,c}\langle G_{g,c,t},x\rangle.
\]
Thus
\[
|\delta z_{t,c}(x)|^2
\le \eta^2 s_t^2 m_{t,c}^2 \|G_{g,c,t}\|^2\|x\|^2
\]
by Cauchy's inequality.  Combining this with the first part and taking expectation over $x$ yields Eq.~\eqref{eq:eh_drift_mask_scale_bound} up to the usual higher-order terms ignored by the local linearization.
\end{proof}

\subsection{Why derivative magnitude is a practical proxy}
\label{app:keff_choice}

The theoretical selection-mismatch term contains the derivative of the selection weight $s$.  NA-LoRA uses $k(z)=|\phi'(z)|$ instead because it is the quantity that directly modulates actual FFN gradients.  Consider a SwiGLU FFN channel
\[
u=W_u x,
\qquad z=W_g x,
\qquad h=\phi(z)\odot u,
\qquad y=W_dh,
\]
and let $\delta=\nabla_y\mathcal{L}$ be the upstream gradient.  With $g=W_d^\top\delta$,
\[
\nabla_u\mathcal{L}=g\odot\phi(z),
\qquad
\nabla_z\mathcal{L}=g\odot u\odot\phi'(z).
\]
Therefore
\begin{equation}
\label{eq:app_grad_up}
\nabla_{W_u}\mathcal{L}=(g\odot\phi(z))x^\top,
\end{equation}
\begin{equation}
\label{eq:app_grad_gate}
\nabla_{W_g}\mathcal{L}=(g\odot u\odot\phi'(z))x^\top.
\end{equation}
Equation~\eqref{eq:app_grad_gate} shows that the gate-projection update is directly gated by $\phi'(z)$.  When $|\phi'(z)|$ is small, gate updates have little immediate effect; when it is large, a small score perturbation can materially change the reweighted channel utility.

Using the self-gated decomposition $\phi(z)=s(z)z$, we have
\[
\phi'(z)=s'(z)z+s(z).
\]
Thus $\phi'(z)$ combines both the change of the decision weight and the current decision weight itself.  This is the relevant quantity for the FFN output because the model optimizes $\phi(z)\odot u$, not $s(z)$ in isolation.  We therefore use
\[
k(z)=|\phi'(z)|,
\qquad
k_{\mathrm{eff}}(c)=\mathbb{E}_{b,s}[k(z_{b,s,c})]
\]
as a stable, architecture-aligned proxy for constructing the temporal-importance mask.

\section{Margin-Based Temporal Bounds}
\label{app:stability_order}

This appendix provides the local temporal bounds referenced in Section~\ref{subsec:temporal_stability_main}.  These bounds justify step scaling as a way to shape one-step movement; they do not imply that NA-LoRA should minimize cumulative raw gate movement.

\begin{remark}[Local sign preservation away from the boundary]
For continuous decision scores, sign changes cannot occur under sufficiently small perturbations when the current margin is bounded away from zero.
\end{remark}
\subsection{Proof of Proposition~\ref{prop:margin_flip_main}}
\begin{proof}
If $\mathrm{sign}(D_{t+1}(x))\neq \mathrm{sign}(D_t(x))$, then either $|D_t(x)|\le \tau$ or $|D_{t+1}(x)-D_t(x)|>\tau$.
Otherwise, if both $|D_t(x)|>\tau$ and $|D_{t+1}(x)-D_t(x)|\le \tau$, then $D_{t+1}(x)$ cannot cross zero and must keep the same sign as $D_t(x)$.
Taking probabilities yields Eq.~\eqref{eq:margin_flip_bound}.
\end{proof}
\subsection{Proof of Corollary~\ref{cor:step_drift_control}}
\begin{proof}
By local Lipschitzness of $D_t(x)$ with respect to the LoRA parameters,
\[
|D_{t+1}(x)-D_t(x)|\le L_D(x)\|\theta_{t+1}-\theta_t\|.
\]
Using the update $\theta_{t+1}=\theta_t-\eta s_t g_t$ gives
\[
|D_{t+1}(x)-D_t(x)|\le \eta s_t L_D(x)\|g_t\|.
\]
Applying Markov's inequality to the squared nonnegative random variable $\eta^2s_t^2L_D(x)^2\|g_t\|^2$ yields Eq.~\eqref{eq:step_drift_bound}.
\end{proof}
\section{Implementation Details}
\label{app:impl}

\subsection{NA-LoRA Training-Time Procedure}
\label{app:nalora_algorithm}

Algorithm~\ref{alg:nalora} summarizes the full training-time procedure of NA-LoRA for non-degenerate SwiGLU-style gates.  It introduces no auxiliary objective and no inference-time modification; instead, it uses gate statistics observed during fine-tuning to improve how the low-rank update is allocated and scaled.  When the activation yields a degenerate coarse statistic, the regime-statistic and scaling lines are skipped and all $s_\star$ are set to $1$.

\begin{algorithm}[ht]
  \caption{NA-LoRA training-time procedure.}
  \label{alg:nalora}
  \begin{algorithmic}
    \STATE {\bfseries Input:} gate pre-activations $z_t$ (per FFN block), keep ratio $\rkeep$, mask sharpness $\beta$, mask EMA $\eta_m$,
    activation-specific regime specification $\tau_z$ when available (code argument \texttt{t}), positive-region penalty $\lambda_{\mathrm{pos}}$,
    step-scale parameters $\{\alpha_\star, s_{\min}^\star, s_{\max}^\star\}_{\star\in\{\mathrm{up},\mathrm{gate},\mathrm{down}\}}$,
    and scale EMA $\beta_s$.
    \STATE {\bfseries State:} mask buffer $m_{t-1}$ and scale buffers $\{s_{\star,t-1}\}$.

    \STATE {\bfseries (1) Temporal-importance channel mask}
    \STATE Compute channel scores $k_{\mathrm{eff}}(c)=\mathbb{E}_{b,s}[|\phi'(z_{t,b,s,c})|]$ for $c\in[d_h]$.
    \STATE Set $\tau_k \leftarrow \mathrm{Quantile}(k_{\mathrm{eff}},1-\rkeep)$.
    \STATE $m_c^{\mathrm{new}}\leftarrow \sigma\!\big(\beta(k_{\mathrm{eff}}(c)-\tau_k)\big)$ for all $c$.
    \STATE Update EMA mask: $m_t\leftarrow \eta_m m_{t-1}+(1-\eta_m)m^{\mathrm{new}}$.

    \STATE {\bfseries (2) Regime-aware statistics}
    \STATE $p_{\mathrm{sup}}\leftarrow \mathbb{E}[\mathbb{I}(z_t<-\tau_z)]$;\;
           $p_{\mathrm{res}}\leftarrow \mathbb{E}[\mathbb{I}(|z_t|\le \tau_z)]$;\;
           $p_{\mathrm{pos}}\leftarrow \mathbb{E}[\mathbb{I}(z_t>\tau_z)]$.
    \STATE $a\leftarrow \mathrm{clamp}(p_{\mathrm{res}}-\lambda_{\mathrm{pos}}p_{\mathrm{pos}},\,0,\,1)$.

    \STATE {\bfseries (3) Projection-wise LoRA step scaling}
    \FOR{$\star\in\{\mathrm{up},\mathrm{gate},\mathrm{down}\}$}
      \STATE $s_\star^{\mathrm{new}}\leftarrow \mathrm{clip}(1+\alpha_\star(2a-1),\ s_{\min}^\star,\ s_{\max}^\star)$.
      \STATE $s_{\star,t}\leftarrow \beta_s s_{\star,t-1}+(1-\beta_s)s_\star^{\mathrm{new}}$.
      \STATE Apply LoRA-only scaling: $\nabla\theta_\star \leftarrow s_{\star,t}\nabla\theta_\star$.
    \ENDFOR

    \STATE {\bfseries (4) Mask application}
    \STATE Apply $m_t$ along the output-channel dimension of gate-related LoRA gradients.

    \STATE {\bfseries Output:} masked and step-scaled LoRA gradients.
  \end{algorithmic}
\end{algorithm}

This section supplements Algorithm~\ref{alg:nalora} with specific implementation details for reproducibility.

\subsection{Hook Locations and Gradient Modification}
NA-LoRA is implemented using PyTorch forward/backward hooks and does not require a custom optimizer.

\paragraph{1. Mask Application (Micro).}
The mask $m_t$ (Eq.~\ref{eq:mask_ema}) is calculated during the forward pass of the FFN block.
It is applied to the \textbf{gradient of the LoRA $B$ matrix} of the \texttt{gate\_proj} module.
Let the LoRA adapter be $W + BA$.
\[
\text{Hook on } \nabla B_{\mathrm{gate}}: \quad (\nabla B_{\mathrm{gate}})_{c, :} \leftarrow m_{t,c}\, (\nabla B_{\mathrm{gate}})_{c, :}.
\]
Here we use the convention $B\in\mathbb{R}^{d_h\times r}$ and $A\in\mathbb{R}^{r\times d}$ for the gate projection.  If an implementation stores LoRA factors in transposed layout, the same channel-wise mask should be applied along the output-channel dimension.
This effectively reallocates capacity away from low-responsiveness channels without modifying the pretrained weights.

\paragraph{2. Step Scaling (Macro).}
The scalars $s_{\mathrm{up}}, s_{\mathrm{gate}}, s_{\mathrm{down}}$ are computed based on the regime statistics of the \emph{same forward pass}.
They are applied to the gradients of \textbf{both $A$ and $B$ matrices} of the corresponding LoRA modules.
\[
\nabla A_{\star} \leftarrow s_{\star} \cdot \nabla A_{\star}, \quad \nabla B_{\star} \leftarrow s_{\star} \cdot \nabla B_{\star}.
\]

\subsection{Hyperparameters}
\label{app:hyperparams}

We report the default hyperparameters used in the Llama experiments, which correspond directly to the \texttt{NALoRAController} implementation.  The parameters are grouped into (i) the \textbf{Micro-level Mask}, which governs channel-selection sparsity and temporal smoothness, and (ii) the \textbf{Macro-level Step Scaling}, which governs regime-based gradient modulation.

Notable settings include a high sharpness ($\beta=20$) for the selection mask to approximate a binary gate, and distinct scaling bounds ($s_{\min}, s_{\max}$) for the \texttt{gate}, \texttt{up}, and \texttt{down} projections to respect their different roles in the FFN mechanism.  Table~\ref{tab:hyperparams} lists the default values used for Llama experiments.

\begin{table}[H]
    \centering
    \caption{Default NA-LoRA hyperparameters for Llama experiments.}
    \label{tab:hyperparams}
    \begin{small}
    \begin{tabular}{llcc}
        \toprule
        \textbf{Component} & \textbf{Parameter / Symbol} & \textbf{Code Arg} & \textbf{Value} \\
        \midrule
        \multicolumn{4}{l}{\textit{\textbf{Micro: Temporal-Importance Mask}}} \\
        \midrule
        \multirow{3}{*}{Selection}
        & Keep Ratio ($\rkeep$) & \texttt{gate\_keep\_ratio} & $0.30$ \\
        & Sharpness ($\beta$) & \texttt{gate\_beta} & $20.0$ \\
        & Mask EMA ($\eta_m$) & \texttt{gate\_ema} & $0.90$ \\
        \midrule
        \multicolumn{4}{l}{\textit{\textbf{Macro: Homogeneity-Dynamics Step Scaling}}} \\
        \midrule
        \multirow{3}{*}{Regime Stats}
        & Regime Threshold ($\tau_z$) & \texttt{t} & $1.27846^\dagger$ \\
        & Positive-region Penalty ($\lambda_{\mathrm{pos}}$) & \texttt{lambda\_ov} & $1.0$ \\
        & Scale EMA ($\beta_s$) & \texttt{beta\_s} & $0.95$ \\
        \midrule
        \multicolumn{4}{l}{\textit{\textbf{Projection-Specific Sensitivity \& Bounds}}} \\
        \midrule
        \multirow{2}{*}{\texttt{gate\_proj}}
        & Sensitivity ($\alpha_{\text{gate}}$) & \texttt{alpha\_gate} & $0.40$ \\
        & Range $[s_{\min}, s_{\max}]$ & \texttt{smin/max\_gate} & $[0.80, 1.50]$ \\
        \midrule
        \multirow{2}{*}{\texttt{up\_proj}}
        & Sensitivity ($\alpha_{\text{up}}$) & \texttt{alpha\_up} & $0.30$ \\
        & Range $[s_{\min}, s_{\max}]$ & \texttt{smin/max\_up} & $[0.80, 1.40]$ \\
        \midrule
        \multirow{2}{*}{\texttt{down\_proj}}
        & Sensitivity ($\alpha_{\text{down}}$) & \texttt{alpha\_down} & $0.20$ \\
        & Range $[s_{\min}, s_{\max}]$ & \texttt{smin/max\_down} & $[0.85, 1.30]$ \\
        \bottomrule
    \end{tabular}
    \end{small}
    \vspace{1mm}
    \begin{flushleft}
    \footnotesize{$^\dagger$ The threshold is specific to SwiGLU and is derived from the characteristic points of $\kappa_\phi(z)=z\phi'(z)/(\phi(z)+\varepsilon)$; see Appendix~\ref{app:local_homogeneity}.  For ReLU-FFN T5 and GELU-based CLIP-ViT, this coarse statistic degenerates and step scaling is disabled.  The code argument remains \texttt{t}, while the paper uses $\tau_z$ to avoid confusion with the optimization step index.}
    \end{flushleft}
\end{table}

\subsection{Computational Overhead}
NA-LoRA introduces negligible overhead.
\begin{itemize}
    \item \textbf{Memory:} Requires storing 3 scalars ($s_\star$) and 1 vector ($m_t$) per layer. No Hessian or second-order states.
    \item \textbf{Compute:} Calculation of $k_{\mathrm{eff}}$ and, when enabled, regime stats requires one pass over the gate activation tensor $z$, which is $\mathcal{O}(B \cdot S \cdot d_h)$. This cost is small relative to the matrix multiplications in the FFN ($d_h \times d$).
\end{itemize}

\section*{LLM Usage}
ChatGPT was used to aid in polishing the writing. Specifically, it was employed to correct grammar, improve readability, and refine the clarity of the prose.

\end{document}